\documentclass{article}

\usepackage{microtype}
\usepackage{graphicx}
\usepackage{subcaption}
\usepackage{booktabs} 
\usepackage{multirow}
\usepackage[inline]{enumitem}

\usepackage{hyperref}

\usepackage[preprint]{icml2026}

\usepackage{amsmath}
\usepackage{amssymb}
\usepackage{mathtools}
\usepackage{amsthm}

\usepackage[capitalize,noabbrev]{cleveref}

\theoremstyle{plain}
\newtheorem{theorem}{Theorem}[section]

\theoremstyle{definition}
\newtheorem{definition}[theorem]{Definition}

\theoremstyle{remark}

\newcommand{\std}[1]{_{\pm \text{\scriptsize #1}}}

\usepackage[textsize=tiny]{todonotes}

\icmltitlerunning{Geometric Mixture-of-Experts with Curvature-Guided Adaptive Routing for Graph Representation Learning}

\begin{document}

\twocolumn[
  \icmltitle{Geometric Mixture-of-Experts with Curvature-Guided Adaptive Routing for Graph Representation Learning}



  \icmlsetsymbol{equal}{*}

  \begin{icmlauthorlist}
    \icmlauthor{Haifang Cao}{sch}
    \icmlauthor{Yu Wang}{sch}
    \icmlauthor{Timing Li}{sch}
    \icmlauthor{Xinjie Yao}{sch}
    \icmlauthor{Pengfei Zhu}{sch}
  \end{icmlauthorlist}

  \icmlaffiliation{sch}{College of Intelligence and Computing, Tianjin University, Tianjin, China}

  \icmlcorrespondingauthor{Haifang Cao}{caohaifang@tju.edu.cn}
  \icmlcorrespondingauthor{Pengfei Zhu}{zhupengfei@tju.edu.cn}

  \icmlkeywords{Machine Learning, ICML}

  \vskip 0.3in
]



\printAffiliationsAndNotice{}  

\begin{abstract}

  Graph-structured data typically exhibits complex topological heterogeneity, making it difficult to model accurately within a single Riemannian manifold. While emerging mixed-curvature methods attempt to capture such diversity, they often rely on implicit, task-driven routing that lacks fundamental geometric grounding. To address this challenge, we propose a Geometric Mixture-of-Experts framework (GeoMoE) that adaptively fuses node representations across diverse Riemannian spaces to better accommodate multi-scale topological structures. At its core, GeoMoE leverages Ollivier-Ricci Curvature (ORC) as an intrinsic geometric prior to orchestrate the collaboration of specialized experts. Specifically, we design a graph-aware gating network that assigns node-specific fusion weights, regularized by a curvature-guided alignment loss to ensure interpretable and geometry-consistent routing. Additionally, we introduce a curvature-aware contrastive objective that promotes geometric discriminability by constructing positive and negative pairs according to curvature consistency. Extensive experiments on six benchmark datasets demonstrate that GeoMoE outperforms state-of-the-art baselines across diverse graph types.

\end{abstract}


\begin{figure}[t!]
  \centering
  \includegraphics[width=0.99\columnwidth]{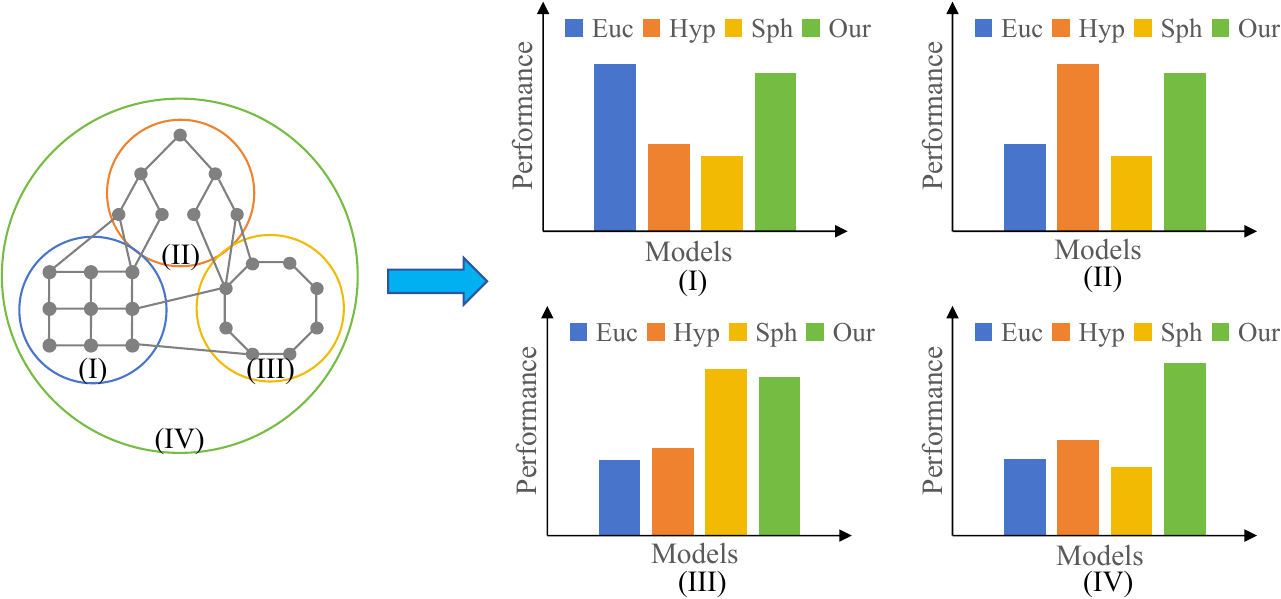}
  \caption{(Left) Topological visualization of real-world graphs (IV), combining grid (I), hierarchy (II), and cycle (III) structures. (Right) Performance of single-geometry models vs. our GeoMoE across these structures.}
  \label{fig: introduction}
\end{figure}


    

\section{Introduction}

Graph representation learning has become a fundamental tool for analyzing relational data in domains ranging from bioinformatics to social networks \cite{atz2021geometric, khoshraftar2024survey, chen2020graph}. A critical challenge lies in the profound geometric heterogeneity of real-world graphs. As visualized in Figure~\ref{fig: introduction}, local structures are rarely uniform, exhibiting a complex mixture of hierarchical expansions (hyperbolic), tight cycles (spherical), and grids (Euclidean) \cite{bronstein2017geometric, defferrard2020deepsphere}. This mismatch between the homogeneous geometric prior assumed by most models and the heterogeneous reality of data inevitably distorts vital structural information and creates a significant bottleneck for representation capacity.

However, most existing graph representation learning methods are built on a single geometric prior, predominantly Euclidean \cite{kipf2016semi, velivckovic2017graph}. Even recent extensions to non-Euclidean settings typically assume a globally shared geometry, such as a fixed hyperbolic or spherical manifold \cite{chami2019hyperbolic, liu2019hyperbolic, zhang2021hyperbolic}. Although these approaches effectively capture certain structural patterns, they implicitly impose a uniform geometric bias over the entire graph. As a result, they struggle to faithfully represent graphs whose local geometry diverges from the chosen global space, leading to suboptimal expressiveness and limited adaptability, as visually demonstrated in \cref{fig: introduction} where single-geometry models only excel on their native substructures.

To capture such geometric heterogeneity, a promising direction is to enable graph representations to jointly exploit multiple geometric spaces. 
Recent works have explored this via mixed-curvature representation learning \cite{cao2025hyperbolic, sun2022self, yang2022dual, sun2024motif}. However, these methods still blend geometries at a global level, applying the same geometric configuration to all nodes, thus failing to achieve fine-grained node-wise adaptation. 
GraphMoRE \cite{guo2025graphmore} introduces a Mixture of Riemannian Experts framework with topology-aware gating for node-specific expert routing.
However, its gating mechanism relies on implicit distortion minimization, lacking supervision from a fundamental geometric prior. This absence of inductive bias results in a routing strategy that is geometrically opaque and prone to expert collapse—where data-driven optimization may gravitate towards trivial solutions rather than adhering to the graph's intrinsic topology.

To bridge these gaps, we propose GeoMoE, a novel geometric mixture-of-experts framework that effectively injects an explicit geometric prior into node-wise expert fusion. The key idea is to utilize Ollivier-Ricci curvature (ORC) \cite{ollivier2009ricci} as a principled descriptor of local graph geometry, accurately capturing neighborhood expansion or contraction and thus indicating whether a region behaves more Euclidean, hyperbolic, or spherical. By treating ORC as an explicit geometric supervision signal, GeoMoE moves beyond purely task-driven routing and enforces geometry-consistent adaptation, aiming to improve both interpretability and robustness under locally varying geometric regimes.


Specifically, GeoMoE orchestrates three specialized experts (Euclidean, hyperbolic, spherical) via a graph-aware gating network.
Crucially, we introduce a gating alignment loss that regularizes the gating distribution toward an ORC-derived target, making expert preferences explicitly consistent with intrinsic discrete geometry. 
To further promote geometric discriminability, we design a curvature-aware contrastive objective that aligns embeddings with their curvature-matched experts, while contrasting them against mismatched experts and inter-node hard negatives, thereby disentangling geometrically distinct structures. 

Our main contributions are summarized as follows: 
\begin{itemize}
    \item We propose GeoMoE, a novel geometric mixture-of-experts framework. By integrating Ollivier-Ricci Curvature as an explicit and informative geometric prior, we transform expert routing into a node-adaptive, geometry-consistent, and interpretable mechanism.
    \item We design a tailored optimization scheme comprising a curvature-guided gating alignment and a curvature-aware contrastive objective. This strategy regularizes the routing logic, ensuring precise topological alignment and enhanced geometric discriminability.
    \item We establish formal theoretical guarantees—including a synergy bound for expert fusion, gating-curvature monotonic consistency, and a mutual information lower bound for contrastive learning—with extensive experiments confirming GeoMoE's state-of-the-art performance and corroborating our theoretical insights.
\end{itemize}


\section{Related work}
\subsection{Ricci Curvature-Driven Graph Learning}
Discrete Ricci curvature—predominantly Ollivier–Ricci curvature(ORC) and Forman–Ricci curvature(FRC)—serves as a principled descriptor of local graph geometry \cite{lin2011ricci, ollivier2009ricci, forman2003bochner}. A major line of research leverages these signals to mitigate structural bottlenecks, such as over-squashing and over-smoothing, via curvature-guided rewiring, sparsification, or Ricci flow-based topology refinement \cite{sun2023deepricci, nguyen2023revisiting, zhang2023ricci, feng2024graph}. Beyond topology refinement, recent advances increasingly integrate curvature as a geometric prior within representation learning. This includes guiding adaptive embedding spaces \cite{sun2023self, zhang2023ricci}
and enhancing domain-specific modeling in biomolecular interaction prediction \cite{shen2024curvature} and knowledge graphs \cite{luo2025local}. Most recently, the paradigm expands into information-theoretic frameworks \cite{fu2025discrete} and Spectro-Riemannian architectures \cite{grover2025spectro}, bridging the gap between discrete geometry and continuous representation theory.

Despite substantial progress, most curvature-driven methods treat Ricci curvature as an auxiliary heuristic rather than a fundamental guiding principle for representation formation. Moreover, they lack mechanisms for fine-grained node-wise adaptation and explicit geometric discriminability.

\subsection{Mixture-of-Experts in Graph Learning}
The Mixture-of-Experts (MoE) \cite{shazeer2017outrageously} paradigm has recently been adapted to graph learning to address structural and semantic heterogeneity by partitioning computation across specialized subnetworks \cite{hu2021graph}. Early explorations introduced MoE for graph-level classification and node representation, aiming to capture diverse structural motifs through expert diversity \cite{errica2021graph, wang2023graph}. This direction has evolved into domain-specific expert designs, such as topology-aware modules for molecular property prediction \cite{kim2023learning, shirasuna2024multi}, link predictors with heterogeneous capacity \cite{ma2024mixture}, and generalizable high-level synthesis modeling \cite{li2025hierarchical}. A more recent and impactful trend focuses on fine-grained node-wise routing, where individual nodes are dynamically assigned to the most suitable experts based on local context—e.g., via adaptive filtering \cite{han2024node}, memory-augmented routers \cite{huang2025graph}, or multi-view alignment \cite{wu2025graph}. Notably, GraphMoRE \cite{guo2025graphmore} extends this paradigm to non-Euclidean settings, introducing a Mixture of Riemannian Experts to mitigate complex geometric heterogeneity through topology-aware manifold routing. 

In spite of these advancements, existing routing mechanisms remain largely implicit, relying on task-driven heuristics that lack explicit alignment with the graph’s intrinsic geometry. 
GeoMoE addresses these limitations by directly internalizing Ollivier-Ricci curvature as a principled supervisory signal. By explicitly guiding expert assignment, GeoMoE ensures that each node is routed to the geometric expert that faithfully reflects its local geometric signature.

\section{Preliminaries}
This section provides the theoretical foundation and mathematical framework necessary for our proposed GeoMoE model. We first establish a consistent notation system and define the Riemannian model spaces utilized by our geometric experts. We then present a rigorous formulation of ORC, which serves as the foundational guiding signal for our gating mechanism and representation formation.

\subsection{Notations}
A graph is formally defined as $\mathcal{G} = (\mathcal{V}, \mathcal{E})$, where $\mathcal{V} = \{v_1, v_2, \dots, v_N\}$ is the set of $N$ nodes and $\mathcal{E} \subseteq \mathcal{V} \times \mathcal{V}$ denotes the set of $m$ edges. The connectivity of $\mathcal{G}$ is characterized by the adjacency matrix $\mathbf{A} \in \{0, 1\}^{N \times N}$, where $A_{ij} = 1$ if $(v_i, v_j) \in \mathcal{E}$ and $0$ otherwise. Each node $v_i$ is associated with a $d$-dimensional feature vector $\mathbf{x}_i \in \mathbb{R}^d$, forming the node feature matrix $\mathbf{X} \in \mathbb{R}^{N \times d}$. The neighborhood of a node $v_i$ is denoted by $\mathcal{N}(v_i) = \{v_j \mid (v_i, v_j) \in \mathcal{E}\}$.

In geometric representation learning, node embeddings are modeled in Riemannian manifolds to reflect the heterogeneous geometric nature of graph structures. We consider three canonical Riemannian manifolds of constant sectional curvature $\kappa \in \mathbb{R}$: the \textbf{Euclidean space} $\mathbb{R}^d$ ($\kappa = 0$); the \textbf{hyperbolic space} $\mathbb{H}_\kappa^d$ ($\kappa < 0$), modeled via the Poincaré ball $\{\mathbf{x} \in \mathbb{R}^d : \|\mathbf{x}\| < 1/\sqrt{-\kappa}\}$; the \textbf{spherical space} $\mathbb{S}_\kappa^d$ ($\kappa > 0$), embedded in $\mathbb{R}^{d+1}$ as $\{\mathbf{x} \in \mathbb{R}^{d+1} : \|\mathbf{x}\| = 1/\sqrt{\kappa}\}$.


For a Riemannian manifold $\mathcal{M}_\kappa$, we denote its exponential map at point $\mathbf{o}$ (typically the origin) as $\exp_{\mathbf{o}} : T_{\mathbf{o}}\mathcal{M}_\kappa \to \mathcal{M}_\kappa$, and its inverse logarithmic map as $\log_{\mathbf{o}} : \mathcal{M}_\kappa \to T_{\mathbf{o}}\mathcal{M}_\kappa$, where $T_{\mathbf{o}}\mathcal{M}_\kappa$ is the associated tangent space at $\mathbf{o}$. The geodesic distance and Riemannian inner product on $\mathcal{M}_\kappa$ are denoted by $d_\kappa(\cdot, \cdot)$ and $\langle \cdot, \cdot \rangle_\kappa$, respectively.

\subsection{Ollivier–Ricci Curvature}
Ollivier-Ricci curvature (ORC)\cite{ollivier2009ricci} is a discrete curvature measure grounded in optimal transport theory, quantifying local graph geometry via Wasserstein distance between neighborhood probability measures. We leverage ORC as a geometric prior to guide node-wise expert routing.

\begin{definition}[Edge-level Ollivier-Ricci Curvature]\cite{ache2019ricci, yang2023kappahgcn}
Let $\mathcal{G} = (\mathcal{V}, \mathcal{E})$ be a locally finite, connected, and simple graph. For any two distinct nodes $v_i, v_j \in \mathcal{V}$ connected by an edge, their edge-level Ollivier-Ricci curvature $\kappa(v_i, v_j)$ is defined as:
\begin{equation}
    \kappa(v_i, v_j) = 1 - \frac{W_1(\mu_{v_i}, \mu_{v_j})}{d(v_i, v_j)},
    \label{eq:edge_orc}
\end{equation}
where $d(v_i, v_j)$ denotes the shortest-path distance between $v_i$ and $v_j$ on $\mathcal{G}$, and $W_1(\mu_{v_i}, \mu_{v_j})$ is the first-order Wasserstein distance (Definition~\ref{def:wassdist}) measuring the discrepancy between two probability measures $\mu_{v_i}$ and $\mu_{v_j}$.
\end{definition}

\begin{definition}[Node-level Ollivier-Ricci Curvature]\cite{ache2019ricci, yang2023kappahgcn}
The node-level Ollivier-Ricci curvature $\kappa(v_i)$ for a node $v_i \in \mathcal{V}$ is defined as the average of edge curvatures over its incident edges:
\begin{equation}
    \kappa(v_i) = \frac{1}{|\mathcal{N}(v_i)|} \sum_{v_j \in \mathcal{N}(v_i)} \kappa(v_i, v_j),
    \label{eq:node_orc}
\end{equation}
where $\kappa(v_i, v_j)$ is given by \cref{eq:edge_orc}. This scalar summary captures the local geometric signature of $v_i$, providing a quantitative basis for the subsequent expert routing. 
\end{definition}



\section{Methodology}\label{Methodology}
Our GeoMoE consists of three synergistic components: (i) a tripartite bank of geometric experts for Euclidean, hyperbolic, and spherical spaces; (ii) a curvature-guided gating alignment mechanism to synchronize routing decisions with intrinsic geometry via a graph-aware router and a gate alignment loss; and (iii) a curvature-aware contrastive learning objective to refine the discriminative power of the fused representations. By integrating ORC as a principled prior, we ensure that the gating mechanism is structurally grounded and the resulting embeddings maintain high geometric fidelity. The overall architecture is illustrated in \cref{figure: framework}.

\begin{figure*}[htbp]
    \centering
    \includegraphics[width=0.95\textwidth]{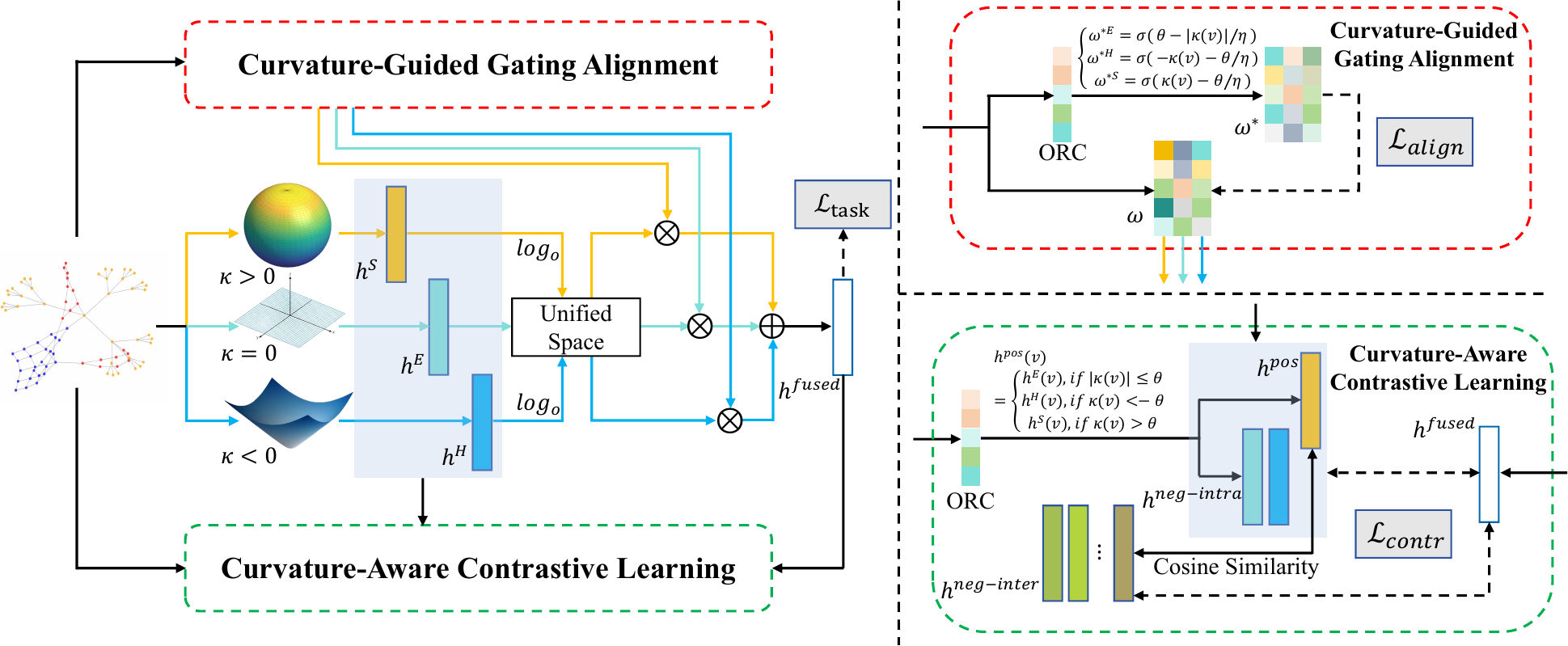}
    \caption{
      An overview of GeoMoE. We first encode the input graph via Euclidean, hyperbolic, and spherical geometric experts, then dynamically fuse their outputs through a curvature-guided gating alignment mechanism. The curvature-aware contrastive learning objective further refines the fused representations to enhance geometric discriminability. 
    }
    \label{figure: framework}
\end{figure*}

\subsection{Geometric Mixture-of-Experts}\label{Geometric Mixture-of-Experts}


Local structures in real-world graphs exhibit profound geometric heterogeneity, with flat, hierarchical, and compact topologies coexisting. To address this, we propose a tri-expert framework that encodes the graph across Euclidean, hyperbolic, and spherical spaces, followed by a dynamic fusion mechanism to effectively reconcile these disparate geometric signatures into a comprehensive representation.

\textbf{Geometric experts.}
Each expert is tailored to the intrinsic properties of its manifold.
Specifically, we instantiate three parallel GNN experts to process signals within their respective optimal geometries: a Euclidean expert for flat patterns, a hyperbolic expert (Poincaré ball) for hierarchies, and a spherical expert for cyclical motifs.
By partitioning the representational task across these specialized encoders, GeoMoE ensures that each structural signature is processed within a space of optimal inductive bias.
Formally, the encoding process for each expert $m \in \{E, H, S\}$ is:
\begin{equation}\label{eq: GNN network}
    \mathbf{h}^m = \text{GNN}^m(\mathbf{X}, \mathbf{A}),
\end{equation}
where $\text{GNN}^m$ denotes the manifold-specific graph neural networks and $\mathbf{h}^m$ represents the resulting node embeddings in the corresponding underlying geometric space.


\textbf{Dynamic feature fusion.}
To achieve effective integration, we employ a gating-weighted summation strategy. For each node $v_i$, embeddings from non-Euclidean manifolds are projected into a unified Euclidean tangent space at the origin $\mathbf{o}$ via logarithmic maps for feature alignment, and fused as:
\begin{equation}\label{eq: fused embedding}
\begin{aligned}
     \mathbf{h}^{\text{fused}}(v_i) 
     & = \omega_i^E \cdot \mathbf{h}^E(v_i) \\
     & + \omega_i^H \cdot \log_{\mathbf{o}}(\mathbf{h}^H(v_i)) + \omega_i^S \cdot \log_{\mathbf{o}}(\mathbf{h}^S(v_i)),
\end{aligned}
\end{equation}
where $\omega_i^E, \omega_i^H, \omega_i^S$ are normalized gating weights such that $\sum_{m \in \{E,H,S\}} \omega_i^m = 1$. 
Notably, these weights are not learned in a black-box fashion. Instead, their distribution is explicitly regularized by the curvature-guided alignment and contrastive learning objectives detailed in \cref{Curvature-Guided Gating Alignment} and \cref{Curvature-Aware Contrastive Learning}. This ensures that the expert assignment is physically grounded in the local geometry of the graph.

Beyond empirical intuition, this synergistic integration aims to overcome the fundamental representation bottlenecks of single-manifold models. By adaptively routing nodes to their curvature-optimal geometric spaces, GeoMoE theoretically expands the expressive capacity of the latent representation compared to any fixed-curvature baseline. To formally justify this, we establish \cref{thm:synergy} on the synergy bound of our geometric mixture-of-experts architecture.

\begin{theorem}[Synergy Bound of Geometric Mixture-of-Experts]\label{thm:synergy}
Let $\ell^m(v_i) = \ell(\mathbf{h}^m(v_i), y(v_i))$ be the loss of expert $m \in \{E, H, S\}$ for node $v_i$, and let $\mathbf{h}^{\text{fused}}(v_i)$ be the fused representation defined in \cref{eq: fused embedding} with gating weights $\boldsymbol{\omega}_i$. Assume the loss function $\ell(\cdot, y)$ is $L$-Lipschitz continuous. Suppose each expert has an advantage region $\Omega_m \subset \mathcal{V}$ with positive probability measure $\mu(\Omega_m) > 0$, such that:
\[
\forall v_i \in \Omega_m,\ \ell^m(v_i) \leq \ell^n(v_i) - \epsilon,\ \forall n \neq m.
\]
Let $e^m$ be the corresponding one-hot vector for expert $m$. If the gating error satisfies $\mathbb{E}[\|\boldsymbol{\omega}_i - e^m\|_1 \mid v_i \in \Omega_m] \leq \delta$ for a sufficiently small positive constant $\delta > 0$, then:
\[
\mathbb{E}[\ell(\mathbf{h}^{\text{fused}}, y)] < \min_{m \in \{E,H,S\}} \mathbb{E}[\ell^m].
\]
\end{theorem}

\begin{proof}
See Appendix~\ref{Proof_synergy} for the complete proof.
\end{proof}


\subsection{Curvature-Guided Gating Alignment}\label{Curvature-Guided Gating Alignment}


Leveraging the property that ORC $\kappa(v)$ characterizes local topology—where positive, negative, and near-zero values indicate spherical, hyperbolic, and Euclidean topologies, respectively—we design an ORC-guided gating alignment mechanism to ensure that expert assignments are physically grounded in the local structural attributes of nodes.

\textbf{Graph-aware gating network.}
Traditional gating mechanisms often rely solely on node features, ignoring the underlying topological structure. This can lead to suboptimal routing where nodes with similar features but distinct neighborhood structures receive identical weights. To address this, we adopt a graph-aware gating network that integrates local topological information through neighborhood aggregation. 

First, a graph convolutional layer is employed to fuse the raw node features with the surrounding structural context:
\begin{equation}
    \mathbf{x}_{\text{agg}} = \text{GCN}_{\text{gate}}(\mathbf{X}, \mathbf{A}),
\end{equation}
where $\mathbf{X}$ and $\mathbf{A}$ are the input feature and adjacency matrices. Then, the expert weights $\mathbf{\omega}$ are generated via a multi-layer perceptron with a temperature-controlled Softmax:
\begin{equation}\label{eq: gating weight}
    \mathbf{\boldsymbol{\omega}} = \text{Softmax}\left(\frac{\text{MLP}(\mathbf{x}_{\text{agg}})}{\tau_g}\right),
\end{equation}
where $\mathbf{\boldsymbol{\omega}} = [\omega^E, \omega^H, \omega^S]^\top$ represents the weights for the Euclidean, hyperbolic, and spherical experts, and $\tau_g$ is the temperature controlling the distribution sharpness.

\textbf{Gating alignment loss.}
To enforce consistency between learned weights and the geometric prior $\kappa(v)$, we introduce an alignment loss based on Kullback-Leibler (KL) divergence. We first construct the geometric target weights $\mathbf{\boldsymbol{\omega}}^* = [\omega^{*E}, \omega^{*H}, \omega^{*S}]^\top$ as soft labels derived from $\kappa(v)$:
\begin{equation}\label{target weight}
\begin{cases} 
    \omega^{*E} = \sigma\left( \frac{\theta - |\kappa(v)|}{\eta} \right), \\
    \omega^{*H} = \sigma\left( \frac{-\kappa(v) - \theta}{\eta} \right), \\
    \omega^{*S} = \sigma\left( \frac{\kappa(v) - \theta}{\eta} \right),
\end{cases}
\end{equation}
where $\theta>0$ is the curvature threshold for region partitioning, $\eta>0$ is a smoothing factor, and $\sigma(\cdot)$ is the Sigmoid function. These components quantify target expertise intensities for the Euclidean (near-zero curvature), hyperbolic (negative curvature), and spherical (positive curvature) domains, respectively. After normalizing the targets such that $\boldsymbol{\omega}^* = \boldsymbol{\omega}^* / \sum_m \omega^{*m}$, the alignment loss is defined as:
\begin{equation}\label{eq: alignment loss}
    \mathcal{L}_{\text{align}} = \text{KL}(\mathbf{\boldsymbol{\omega}}^* \parallel \mathbf{\boldsymbol{\omega}}).
\end{equation}

Theoretically, minimizing the alignment loss $\mathcal{L}_{\text{align}}$ enforces a strict monotonic consistency between the gating weights and the node-level ORC values, which guarantees the structural rationality of the expert allocation. We formally establish this pivotal property through the following theorem.


\begin{theorem}[Monotonic Consistency Between Gating Weights and ORC]\label{thm:Monotonic}
Let $\kappa(v)$ be the node-level ORC of node $v$. Under the alignment loss $\mathcal{L}_{\text{align}}$ with target weights defined in \cref{target weight}, the learned gating weights at convergence $\boldsymbol{\omega}(v) = [\omega^E(v), \omega^H(v), \omega^S(v)]^\top$ satisfy the following regional monotonicity properties:
\begin{enumerate*}[label=(\roman*)]
    \item \textbf{Spherical Domain:} On $\Omega_+ = \{ v \in \mathcal{V} \mid \kappa(v) \geq \theta \}$, the spherical weight $\omega^S$ is strictly increasing with respect to $\kappa(v)$, i.e., $\frac{\partial \omega^S(v)}{\partial \kappa(v)} > 0$;
    \item \textbf{Hyperbolic Domain:} On $\Omega_- = \{ v \in \mathcal{V} \mid \kappa(v) \leq -\theta \}$, the hyperbolic weight $\omega^H$ is strictly decreasing with respect to $\kappa(v)$, i.e., $\frac{\partial \omega^H(v)}{\partial \kappa(v)} < 0$;
    \item \textbf{Euclidean Domain:} On $\Omega_0 = \{ v \in \mathcal{V} \mid |\kappa(v)| \leq \theta \}$, the Euclidean weight $\omega^E$ is strictly decreasing with respect to the curvature magnitude $|\kappa(v)|$, i.e., $\frac{\partial \omega^E(v)}{\partial |\kappa(v)|} < 0$.
\end{enumerate*}
\end{theorem}

\begin{proof}
See Appendix~\ref{Proof_monotonic} for the complete proof.
\end{proof}

\subsection{Curvature-Aware Contrastive Learning}\label{Curvature-Aware Contrastive Learning}

While curvature-guided alignment ensures geometric consistency, it may not guarantee the discriminative power.
To address this, we propose a curvature-aware contrastive objective, employing ORC values as geometric priors to assign positive samples to the most compatible geometric experts and construct informative negative sets. This enhances the discriminability and robustness of the fused representations.

\textbf{Positive sample selection.}
Based on the geometric semantics of ORC values $\kappa(v)$, we define an ORC-driven hard selection mechanism to identify the optimal expert representation as the positive sample $\mathbf{h}^{\text{pos}}(v)$. This systematic selection logic remains strictly consistent with the ideal weighting scheme previously defined in \cref{Curvature-Guided Gating Alignment}:
\begin{equation}
    \mathbf{h}^{\text{pos}}(v) = 
    \begin{cases} 
        \mathbf{h}^E(v), & \text{if } |\kappa(v)| \le \theta; \\
        \mathbf{h}^H(v), & \text{if } \kappa(v) < -\theta; \\
        \mathbf{h}^S(v), & \text{if } \kappa(v) > \theta .
    \end{cases}
\end{equation}
Here, $\theta > 0$ denotes the identical curvature threshold used for region partitioning as in \cref{Curvature-Guided Gating Alignment}. This formulation explicitly ensures that the positive samples are highly congruent with the primary gating alignment objectives.

\textbf{Dual-strategy negative construction.}
To further sharpen the decision boundaries, we design a dual-strategy for constructing negative sets, comprising \textit{intra-node} and \textit{inter-node} negative samples. For intra-node negatives, we select the non-optimal expert representations of the current node $v$, which encourages the model to differentiate between distinct geometric manifolds. Simultaneously, we identify inter-node hard negatives by mining other nodes that exhibit high similarity to $\mathbf{h}^{\text{pos}}(v)$ but possess distinct geometric properties. Specifically, we compute the cosine similarity between the positive embedding $\mathbf{h}^{\text{pos}}(v)$ and the fused representations $\mathbf{h}^{\text{fused}}(u)$ of all nodes $u \in \mathcal{V} \setminus \{v\}$, selecting the top-$(K-2)$ most similar embeddings as hard negatives.

\textbf{Contrastive loss.}
To formalize the training objective, we define $s_{v, \text{pos}} = \text{sim}(\mathbf{h}^{\text{fused}}(v), \mathbf{h}^{\text{pos}}(v))$ and $s_{v, j} = \text{sim}(\mathbf{h}^{\text{fused}}(v), \mathbf{h}^{\text{fused}}(v_j))$ as the cosine similarity scores for the positive pair and the $j$-th negative pair, respectively. Based on the standard InfoNCE framework \cite{oord2018representation}, the contrastive loss is mathematically formulated as:
\begin{equation}\label{eq: contrastive loss}
    \mathcal{L}_{\text{contr}} = - \frac{1}{N} \sum_{v \in \mathcal{V}} \log \frac{e^{s_{v, \text{pos}} / \tau_c}}{e^{s_{v, \text{pos}} / \tau_c} + \sum_{j \in \mathcal{N}^-(v)} e^{s_{v, j} / \tau_c}},
\end{equation}
where $\mathcal{N}^-(v)$ denotes the set of negative samples of node $v$, and $\tau_c$ is the contrastive temperature parameter. 


Curvature-aware contrastive learning maximizes the mutual information between node representations and their underlying geometric signatures. To characterize this information-theoretic guarantee, we establish the following theorem.



\begin{theorem}[Geometry-Consistent Mutual Information Maximization]\label{thm:Mutual_information}
Following the InfoNCE principle \cite{oord2018representation, poole2019variational}, let $\mathbf{h}^{\text{fused}}$ and $\mathbf{h}^*$ denote the random variables for the fused representation and the geometric prior embedding (derived from the ORC-guided positive selection), respectively. For a batch containing one positive sample and $K$ negative samples drawn independently from the marginal distribution $p(\mathbf{h}^*)$, the contrastive loss $\mathcal{L}_{\text{contr}}$ serves as a lower bound estimator:
\begin{equation}
I(\mathbf{h}^{\text{fused}}; \mathbf{h}^*) \geq \log(K+1) - \mathcal{L}_{\text{contr}}.
\end{equation}
Consequently, minimizing $\mathcal{L}_{\text{contr}}$ maximizes the variational lower bound of the mutual information between the node representations and their curvature-driven geometric priors, thereby effectively promoting geometric consistency and mitigating representational collapse in the latent space.
\end{theorem}

\begin{proof}See Appendix~\ref{Proof_mutual} for the detailed derivation.
\end{proof}

\subsection{Training Objective}\label{Training Objective}
The total loss of the model is formulated as a weighted linear combination of the supervised loss, the gating alignment loss, and the contrastive loss. These terms directly correspond to the three core optimization objectives: downstream task fitting, geometric consistency constraints, and enhancement of geometric discriminability, respectively: 
\begin{equation}\label{eq: total loss}
    \mathcal{L}_{\text{total}} = \alpha \mathcal{L}_{\text{task}} + \beta \mathcal{L}_{\text{align}} + \gamma \mathcal{L}_{\text{contr}},
\end{equation}

where $\mathcal{L}_{\text{task}}$ denotes the task-specific supervised loss (e.g., cross-entropy), and the non-negative hyperparameters $\alpha, \beta, \gamma$ govern the trade-off between the predictive accuracy and the geometric regularization terms.
We provide the detailed training algorithm in \cref{alg:geomoe} of the Appendix.


\subsection{Complexity and Scalability Analysis}
\label{sec:complexity}

A primary bottleneck in curvature-driven learning is the computational cost. To address this, GeoMoE adopts a strategy that decouples geometric characterization from model optimization. Specifically, ORC computation is treated as an offline pre-processing step. While exact solvers are employed for benchmark datasets, our framework can seamlessly integrate with near linear-time approximations \cite{ni2019community} for scalability. By utilizing neighborhood overlap measures, the pre-computation complexity reduces from $\mathcal{O}(|\mathcal{E}|\Delta^3 \log \Delta)$—where $|\mathcal{E}|$ is the edge count and $\Delta$ is the maximum degree—to a manageable $\mathcal{O}(|\mathcal{E}|\Delta)$, which is computationally efficient for sparse graphs. Crucially, during the online phase, the curvature-guided gating incurs only $\mathcal{O}(1)$ overhead per node, maintaining a training complexity of $\mathcal{O}(|\mathcal{E}|)$ asymptotically equivalent to standard GCNs. This ensures GeoMoE is theoretically scalable to large-scale graphs. Further derivations are detailed in Appendix~\ref{app:complexity}.

\begin{table*}
    \centering
    \caption{Accuracy ($\%\pm$ standard deviation) for Cora, CiteSeer, PubMed, Photo and WikiCS and F1 score ($\%\pm$ standard deviation) for Airport for node classification. Scores are averaged over ten runs. The best results are boldfaced while the runner-ups are underlined.}
    \label{table:main results}
    \begin{tabular}{ccccccc}
    \toprule
    Method & Cora & CiteSeer & PubMed & Photo & WikiCS & Airport \\
    \midrule
    GCN \cite{kipf2016semi} & $82.02 \std{0.45}$ & $70.24 \std{0.85}$ & $78.45 \std{0.36}$ & $92.84 \std{0.25}$ & $74.29 \std{1.06}$ & $83.19 \std{1.76}$ \\
    GAT \cite{velivckovic2017graph} & $82.72 \std{0.69}$ & $71.45 \std{1.15}$ & $77.23 \std{0.51}$ & $92.70 \std{0.55}$ & $76.55 \std{0.88}$ & $86.10 \std{1.39}$ \\
    SGC \cite{wu2019simplifying} & $81.91 \std{0.63}$ & $71.95 \std{0.55}$ & $78.60 \std{0.24}$ & $91.96 \std{1.23}$ & $75.54 \std{1.25}$ & $84.07 \std{1.91}$ \\
    GraphCON \cite{rusch2022graph} & $82.54 \std{1.25}$ & $71.29 \std{0.96}$ & $78.66 \std{1.23}$ & $91.94 \std{1.12}$ & $76.64 \std{0.59}$ & $79.25 \std{2.51}$ \\
    NodeFormer \cite{wu2022nodeformer} & $82.34 \std{0.85}$ & $72.47 \std{1.27}$ & $79.13 \std{1.10}$ & $93.46 \std{0.62}$ & $73.27 \std{0.87}$ & $82.66 \std{0.69}$ \\
    ACMP \cite{wang2022acmp} & $82.65 \std{0.78}$ & $73.01 \std{1.14}$ & $79.12 \std{0.54}$ & $92.24 \std{1.71}$ & $77.12 \std{0.58}$ & $88.96 \std{1.74}$ \\
    SGFormer \cite{wu2023sgformer} & $83.01 \std{0.75}$ & $72.35 \std{0.54}$ & $\underline{79.25} \std{0.65}$ & $90.74 \std{1.51}$ & $74.72 \std{0.76}$ & $91.67 \std{1.21}$ \\
    RNCGLN \cite{zhu2024robust} & $82.58 \std{0.66}$ & $\underline{73.60} \std{0.68}$ & $78.64 \std{0.87}$ & $91.71 \std{0.28}$ & $77.65 \std{0.85}$ & $90.94 \std{0.77}$ \\
    DGMAE \cite{zheng2025discrepancy} & $\underline{84.24} \std{0.58}$ & $72.91 \std{0.38}$ & $79.21 \std{0.28}$ & $93.06 \std{0.13}$ & $\underline{78.81} \std{0.34}$ & $90.58 \std{0.95}$ \\
    \midrule
    HGNN \cite{liu2019hyperbolic} & $78.99 \std{0.61}$ & $70.20 \std{0.61}$ & $76.94 \std{1.12}$ & $91.82 \std{0.35}$ & $74.89 \std{0.42}$ & $84.87 \std{2.09}$ \\
    HGCN \cite{chami2019hyperbolic} & $78.41 \std{0.77}$ & $68.29 \std{0.95}$ & $77.03 \std{0.49}$ & $93.19 \std{0.29}$ & $76.67 \std{0.44}$ & $88.82 \std{1.66}$ \\
    GIL \cite{zhu2020graph} & $82.49 \std{0.77}$ & $71.21 \std{0.93}$ & $77.35 \std{0.56}$ & $93.92 \std{0.18}$ & $77.78 \std{0.85}$ & $91.21 \std{1.72}$ \\
    HGAT \cite{zhang2021hyperbolic} & $79.71 \std{0.95}$ & $69.41 \std{0.78}$ & $75.56 \std{0.67}$ & $93.38 \std{0.45}$ & $77.62 \std{0.49}$ & $89.25 \std{0.96}$ \\
    LGCN \cite{zhang2021lorentzian} & $78.93 \std{0.79}$ & $68.59 \std{0.64}$ & $78.08 \std{0.65}$ & $94.21 \std{0.43}$ & $76.58 \std{0.60}$ & $88.53 \std{1.26}$ \\
    F-HNN \cite{chen2022fully} & $81.04 \std{0.78}$ & $71.12 \std{0.66}$ & $77.66 \std{1.04}$ & $94.55 \std{0.52}$ & $78.07 \std{0.51}$ & $91.85 \std{0.86}$ \\
    MotifRGC \cite{sun2024motif} & $81.65 \std{1.03}$ & $72.56 \std{1.21}$ & $78.33 \std{0.85}$ & $93.78 \std{0.93}$ & $78.12 \std{0.72}$ & $92.31 \std{1.85}$ \\
    LRN \cite{he2025lorentzian} & $78.34 \std{1.16}$ & $67.10 \std{1.19}$ & $77.24 \std{0.67}$ & $\underline{94.58} \std{0.39}$ & $77.98 \std{0.42}$ & $92.39 \std{0.72}$ \\
    GraphMoRE \cite{guo2025graphmore} & $82.85 \std{0.97}$ & $71.64 \std{0.85}$ & $78.65 \std{0.66}$ & $94.07 \std{1.05}$ & $78.27 \std{0.68}$ & $\underline{92.76} \std{0.94}$ \\
    \midrule
    GeoMoE(ours) & $\mathbf{84.75} \std{0.78}$ & $\mathbf{74.72} \std{0.62}$ & $\mathbf{80.13} \std{0.58}$ & $\mathbf{95.18} \std{1.04}$ & $\mathbf{79.65} \std{0.53}$ & $\mathbf{93.56} \std{1.21}$ \\

    \bottomrule
    \end{tabular}
\end{table*}

\section{Experiments}\label{experiments}
\subsection{Experimental Settings}\label{Experimental Settings}

\textbf{Datasets.} 
We evaluate GeoMoE on six benchmark datasets with diverse structural topologies: three citation networks (Cora, CiteSeer, PubMed)\cite{sen2008collective, namata2012query}, a co-purchase graph (Photo)\cite{shchur2018pitfalls}, a reference graph (WikiCS)\cite{mernyei2020wiki}, and a flight network (Airport)\cite{chami2019hyperbolic}.
These datasets collectively span a wide range of graph geometries, with detailed statistics provided in \cref{table:Dataset} of the Appendix~\ref{Experimental Settings appendix}.

\textbf{Baselines.}
We compare our method against a wide range of state-of-the-art baselines, which can be broadly categorized into two groups based on the underlying geometric spaces:
(1) Euclidean-based methods, which model graph representations exclusively in flat Euclidean space;
(2) Riemannian methods, which operate in curved Riemannian manifolds (hyperbolic, spherical, or mixed-curvature spaces) to capture latent hierarchical, cyclic, or heterogeneous structures.

\textbf{Implementation Details.}
We adopt standard splits for three citation networks, official splits for WikiCS, and random splits for Photo (60/20/20\%) and Airport (70/15/15\%). For a fair comparison, all models are optimized via grid search and evaluated over 10 independent runs with a fixed hidden dimension of 16. Our GeoMoE framework incorporates GCN, HGAT, and Spherical-GCN as the Euclidean, hyperbolic, and spherical experts, respectively. All experiments are conducted on an NVIDIA RTX 3090 GPU, with comprehensive implementation details provided in Appendix~\ref{Experimental Settings appendix}.


\begin{table*}[t]
\centering
\caption{Ablation study. In the table, ``\checkmark'' denotes GeoMoE with the component. The best results are boldfaced.}
\label{tab:ablation_optimized}
\begin{tabular}{l ccc cc cccccc}
\toprule
& \multicolumn{3}{c}{\textbf{Experts}} & \multicolumn{2}{c}{\textbf{Losses}} & \multicolumn{6}{c}{\textbf{Datasets}} \\
\cmidrule(lr){2-4} \cmidrule(lr){5-6} \cmidrule(lr){7-12}
Variant & Euc. & Hyp. & Sph. & $\mathcal{L}_{\text{align}}$ & $\mathcal{L}_{\text{contr}}$ & Cora & CiteSeer & PubMed & Photo & WikiCS & Airport \\
\midrule
(a) & \checkmark & & & & & 82.02 & 70.24 & 78.45 & 92.84 & 74.29 & 83.19 \\
(b) & & \checkmark & & & & 79.71 & 69.41 & 75.56 & 93.38 & 77.62 & 89.25 \\
(c) & & & \checkmark & & & 76.68 & 70.57 & 76.85 & 91.61 & 75.85 & 85.63 \\
(d) & \checkmark & \checkmark & & \checkmark & \checkmark & 82.97 & 71.12 & 78.73 & 93.88 & 78.13 & 91.97 \\
(e) & \checkmark & & \checkmark & \checkmark & \checkmark & 82.75 & 71.37 & 78.84 & 93.42 & 75.76 & 87.28 \\
(f) & & \checkmark & \checkmark & \checkmark & \checkmark & 83.47 & 72.89 & 77.19 & 93.76 & 77.85 & 90.31 \\
\midrule
(g) & \checkmark & \checkmark & \checkmark & & & 82.23 & 70.91 & 78.45 & 93.52 & 77.99 & 91.40 \\
(h) & \checkmark & \checkmark & \checkmark & \checkmark & & 83.02 & 72.90 & 78.59 & 93.79 & 78.23 & 91.77 \\
(i) & \checkmark & \checkmark & \checkmark & & \checkmark & 82.77 & 73.84 & 78.81 & 94.25 & 78.45 & 92.19 \\
\midrule
GeoMoE & \checkmark & \checkmark & \checkmark & \checkmark & \checkmark & \textbf{84.75} & \textbf{74.72} & \textbf{80.13} & \textbf{95.18} & \textbf{79.65} & \textbf{93.56} \\
\bottomrule
\end{tabular}
\end{table*}

\subsection{Main Results}\label{Main Results}
\textbf{Bridging the Euclidean-Riemannian gap.}
GeoMoE establishes new state-of-the-art performance by consistently surpassing both Euclidean and Riemannian baselines across all datasets. On citation networks where Euclidean GNNs typically dominate, GeoMoE outperforms the strongest baseline methods by 0.51\%--1.12\%. More significantly, on datasets favoring non-Euclidean geometry, GeoMoE achieves superior results, surpassing the leading mixture model, GraphMoRE, by an impressive average margin of 1.28\%. This performance highlights the efficacy of our design: distinct from GraphMoRE's implicit routing or fixed-manifold approaches, GeoMoE leverages explicit ORC supervision to ensure precise topological alignment, effectively bridging the performance gap between these competing paradigms.

\textbf{Adaptive superiority across geometric regimes.}
GeoMoE demonstrates the strongest gains on graphs with diverse geometric characteristics. On CiteSeer—a citation network exhibiting mixed structural motifs—GeoMoE achieves the largest improvement, outperforming the best baseline (RNCGLN) by $1.12\%$. On WikiCS, which combines hierarchical article links with dense community structures, GeoMoE surpasses the top baseline model (DGMAE) by $0.84\%$. Even on relatively homogeneous graphs such as the hierarchical Airport and the community-dense Photo, it consistently improves over specialized baselines by $0.80\%$ and $0.60\%$, respectively, by leveraging complementary geometric representations. This consistent advantage confirms that adaptive multi-geometry routing optimally represents the varied curvature regimes found in real-world graphs.

\subsection{Ablation Study}

To dissect the contributions of the core components in GeoMoE, we conduct a series of comprehensive ablation experiments by constructing systematic variants that omit specific experts or regularization losses, as summarized in \cref{tab:ablation_optimized}.

\textbf{Synergy of multi-geometry experts.}
The full three-expert GeoMoE consistently outperforms all single- and dual-expert variants. Single-geometry baselines (variants \textit{a}--\textit{c}) achieve competitive but limited performance, with each excelling only on graphs matching its geometric bias (e.g., hyperbolic expert on Airport: 89.25\%, Euclidean on Cora: 82.02\%). Dual-expert combinations (variants \textit{d}--\textit{f}) show partial synergy but remain suboptimal, particularly on datasets with mixed structural motifs like CiteSeer and WikiCS. In contrast, GeoMoE achieves the best results everywhere, confirming that Euclidean, hyperbolic, and spherical geometries provide complementary representational capacities that collectively capture a wider range of graph structures.

\textbf{Efficacy of geometric regularization components.}
The geometric gating alignment loss ($\mathcal{L}_{\text{align}}$) and contrastive loss ($\mathcal{L}_{\text{contr}}$) are critical for optimizing performance. Without either loss (variant g), performance is unstable—strong on Photo but underperforming on Cora ($82.23\%$) and CiteSeer ($70.91\%$). Adding $\mathcal{L}_{\text{align}}$ alone (variant h) boosts CiteSeer accuracy by $2.0\%$ to $72.90\%$ via improved routing alignment. Incorporating $\mathcal{L}_{\text{contr}}$ independently (variant i) enhances discriminative power, increasing PubMed accuracy to $78.81\%$ and Airport F1 to $92.19\%$. 
Mechanistically, $\mathcal{L}_{\text{align}}$ anchors routing to topological structures while $\mathcal{L}_{\text{contr}}$ ensures the global distinctiveness of multi-manifold representations.
GeoMoE, integrating both objectives, achieves optimal performance across all datasets. This synergy confirms that explicit geometric alignment and contrastive refinement are indispensable for fully activating the representational potential of a multi-geometry mixture-of-experts.

\begin{figure}[t!]
  \centering
  \includegraphics[width=0.98\columnwidth]{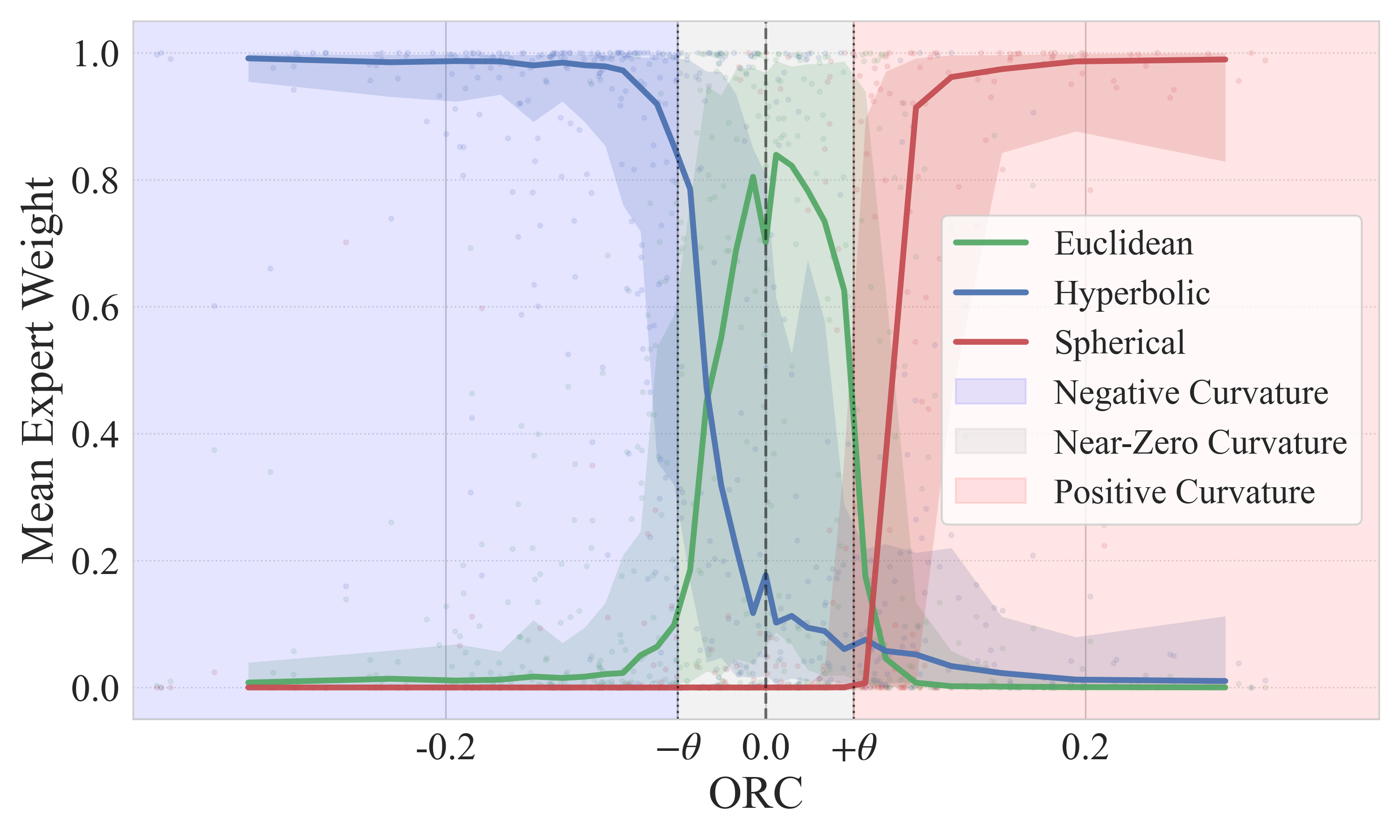}
  \caption{Consistency between gating weights and ORC on Photo.}
  \label{fig:orc_gating_correlation}
\end{figure}

\subsection{Further Analysis}\label{Further Analysis}
To further understand the underlying mechanism of GeoMoE, we first evaluate the routing consistency with geometric priors across diverse curvature regimes, followed by a sensitivity study on the critical curvature threshold $\theta$. 

\textbf{Consistency between gating weights and ORC.}
We assess whether GeoMoE’s routing behavior aligns with geometric priors. As shown in \cref{fig:orc_gating_correlation}, expert utilization on Photo exhibits a clear curvature-aligned pattern: the hyperbolic expert dominates negatively curved regions, whereas the spherical expert prevails as ORC becomes positive. Crucially, the Euclidean weight peaks at $\kappa \approx 0$ and decays with curvature magnitude, confirming its specialization in flat topologies. The sharp expert transitions observed around $\pm\theta$ demonstrate the gate network's discriminative sensitivity, confirming a curvature-aware manifold partitioning rather than a diffuse mixing strategy. These results empirically support that our alignment objective effectively internalizes ORC as a routing signal, consistent with \cref{thm:Monotonic}.


\begin{figure}[t!]
  \centering
  \includegraphics[width=0.95\columnwidth]{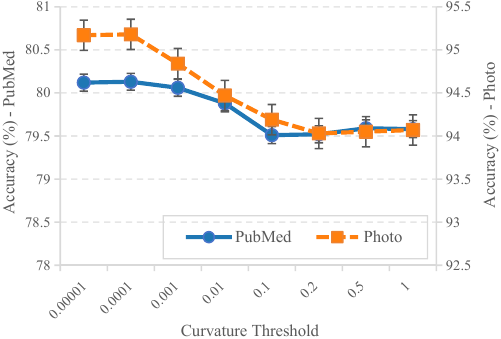}
  \caption{Sensitivity to curvature threshold $\theta$.}
  \label{fig:curvature_threshold}
\end{figure}


\textbf{Sensitivity to curvature threshold $\theta$.}
To validate the robustness of our ORC-guided routing mechanism, we analyze the model's sensitivity to the curvature threshold $\theta$.
As shown in \cref{fig:curvature_threshold}, GeoMoE peaks at $\theta = 10^{-4}$ on both PubMed and Photo, indicating that subtle curvature signals are crucial for routing. Performance remains stable on PubMed as $\theta$ increases, consistent with its Euclidean bias, while accuracy on Photo exhibits a sharper decline beyond $\theta = 10^{-3}$, confirming that forcing Riemannian-inclined nodes into Euclidean space incurs a geometric distortion penalty. 
Essentially, $\theta$ acts as a structural filter that modulates the transition between non-Euclidean and Euclidean representational regimes. 
The consistent performance across several orders of magnitude ($[10^{-5}, 10^{-2}]$) demonstrates that our mechanism is highly robust to threshold selection, with a small default value being effective for diverse graph topologies.

\section{Conclusion}\label{Conclusion}

We present GeoMoE, a novel geometric mixture-of-experts framework that leverages Ollivier–Ricci curvature to enable node-wise adaptive routing across Euclidean, hyperbolic, and spherical spaces. By explicitly aligning gating decisions and contrastive learning with intrinsic graph geometry, GeoMoE ensures both interpretability and discriminability. Theoretical analysis validates the geometric consistency of our design, while extensive experiments demonstrate consistent superiority over strong baselines. This work bridges discrete curvature theory with multi-geometry representation learning, offering a principled path toward more faithful graph embeddings. Future work will explore the extension of curvature-aware routing to complex dynamic networks.

\section*{Impact Statement}
This paper presents GeoMoE, a framework designed to improve the adaptivity and interpretability of graph representation learning through intrinsic geometric priors. As this work focuses on fundamental algorithmic advancements in machine learning, it has broad potential applications in fields such as social network analysis and bioinformatics. There are many potential societal consequences of our work, none of which we feel must be specifically highlighted here.

\nocite{langley00}

\bibliography{example_paper}
\bibliographystyle{icml2026}

\newpage
\clearpage
\appendix

\section{Supplementary Preliminary}\label{Supplementary Preliminary}

\begin{definition}[First-Order Wasserstein Distance]\label{def:wassdist}
Let $\mu_i$ and $\mu_j$ be two probability measures defined over the node set $\mathcal{V}$. The first-order Wasserstein distance $W_1(\mu_i, \mu_j)$ is defined as:
\begin{equation}
    W_1(\mu_i, \mu_j) = \inf_{\pi \in \Pi(\mu_i, \mu_j)} \sum_{v_k \in \mathcal{V}, v_l \in \mathcal{V}} \pi(v_k, v_l) \, d(v_k, v_l),
    \label{eq:wasserstein}
\end{equation}
where $\Pi(\mu_i, \mu_j)$ is the set of all transport plans from $\mu_i$ to $\mu_j$. A transport plan $\pi : \mathcal{V} \times \mathcal{V} \to [0,1]$ satisfies the marginal constraints:
\begin{equation}
    \sum_{v_l \in \mathcal{V}} \pi(v_k, v_l) = \mu_i(v_k)
    \quad \text{and} \quad
    \sum_{v_k \in \mathcal{V}} \pi(v_k, v_l) = \mu_j(v_l).
    \label{eq:marginal}
\end{equation}
\end{definition}

\begin{definition}[Local Probability Measure]\label{Local Probability Measure}
Given a graph $\mathcal{G} = (\mathcal{V}, \mathcal{E})$, for a node $v_i \in \mathcal{V}$, let $d_{v_i}$ denote its degree and $\mathcal{N}(v_i)$ its neighborhood set. A parameter $p \in [0, 1]$ controls the weight assigned to the center node itself. The local probability measure $\mu_{v_i}$ over $\mathcal{V}$ is defined as:
\begin{equation}
    \mu_{v_i}(v) =
    \begin{cases}
        p, & \text{if } v = v_i, \\
        \frac{1-p}{d_{v_i}}, & \text{if } v \in \mathcal{N}(v_i), \\
        0, & \text{otherwise}.
    \end{cases}
    \label{eq:local_measure}
\end{equation}
\end{definition}

\section{Theoretical Analysis}\label{Theoretical Analysis}

\subsection{Proof of \cref{thm:synergy}}\label{Proof_synergy}


\begin{proof}
Let $I^*(v_i) = \arg\min_{m \in \{E,H,S\}} \ell^m(v_i)$ be the oracle expert selector \cite{jordan1994hierarchical}, and define:
\begin{equation}
   \mathbf{h}^*(v_i) = \mathbf{h}^{I^*(v_i)}(v_i), 
\end{equation}
\begin{equation}
   \ell^*(v_i) = \ell(\mathbf{h}^*(v_i), y(v_i)) = \min_m \ell^m(v_i). 
\end{equation}
By definition, for any expert $m$, we have:
\begin{equation}
    \mathbb{E}[\ell^*] \leq \mathbb{E}[\ell^m].
\end{equation}
Furthermore, by the advantage region assumption, there exists $\eta_0 > 0$ such that:
\begin{equation}\label{eq:step1-oracle-bound}
\mathbb{E}[\ell^*] < \min_m \mathbb{E}[\ell^m] - \eta_0.
\end{equation}
Assume the loss function $\ell(\cdot, y)$ is $L$-Lipschitz in its first argument \cite{bubeck2015convex}:
\begin{equation}
\begin{split}
    \left|\ell(\mathbf{h}^{\text{fused}}(v_i),y(v_i)) - \ell(\mathbf{h}^*(v_i),y(v_i))\right| \\
\le L\cdot\|\mathbf{h}^{\text{fused}}(v_i)-\mathbf{h}^*(v_i)\|, 
\end{split}
\end{equation}
where $L$ is Lipschitz constant of the loss function $\ell$. Then:
\begin{equation}
    \ell(\mathbf{h}^{\text{fused}}(v_i), y(v_i)) \leq \ell^*(v_i) + L \cdot \| \mathbf{h}^{\text{fused}}(v_i) - \mathbf{h}^*(v_i) \|.
\end{equation}
Taking expectations yields:
\begin{equation}\label{eq:step2-lipschitz-bound}
\mathbb{E}[\ell(\mathbf{h}^{\text{fused}}, y)] \leq \mathbb{E}[\ell^*] + L \cdot \mathbb{E} \| \mathbf{h}^{\text{fused}} - \mathbf{h}^* \| ].
\end{equation}


Recall:
\begin{equation}
    \mathbf{h}^{\text{fused}}(v_i) = \sum_m \omega_i^m \mathbf{h}^m(v_i).
\end{equation}
Thus:
\begin{equation}
\begin{split}
    &\| \mathbf{h}^{\text{fused}}(v_i) - \mathbf{h}^*(v_i) \| \\
    &= \left\| \sum_m \omega_i^m (\mathbf{h}^m(v_i) - \mathbf{h}^*(v_i)) \right\| \\
    &\leq \sum_m \omega_i^m \| \mathbf{h}^m(v_i) - \mathbf{h}^*(v_i) \|.
\end{split}
\end{equation}
Since all embeddings are projected into a common Euclidean tangent space via logarithmic maps and the hyperbolic and spherical spaces are compact, there exists a finite constant $M < \infty$ such that:
\[
\| \mathbf{h}^m(v_i) - \mathbf{h}^n(v_i) \| \leq M \quad \text{for all } m,n \in \{E,H,S\}.
\]
When $m = I^*(v_i)$, we have $\mathbf{h}^m(v_i) = \mathbf{h}^*(v_i)$, so that term vanishes. Therefore:
\begin{equation}
\begin{split}
    & \sum_m \omega_i^m \| \mathbf{h}^m(v_i) - \mathbf{h}^*(v_i) \| \\
    & = \sum_{m\neq I^*(v_i)} \omega_i^m \| \mathbf{h}^m(v_i) - \mathbf{h}^*(v_i) \|  \\
    & \leq M \sum_{m \neq I^*(v_i)} \omega_i^m.
\end{split}
\end{equation}
Note that:
\begin{equation}
\begin{split}
    \| \boldsymbol{\omega}_i - e^{I^*(v_i)} \|_1 = 2\big(1-\omega_i^{I^*(v_i)}\big) = 2 \sum_{m \neq I^*(v_i)} \omega_i^m.
\end{split}
\end{equation}
so
\begin{equation}
    \sum_{m \neq I^*(v_i)} \omega_i^m
= \frac{1}{2}\| \boldsymbol{\omega}_i - e^{I^*(v_i)} \|_1.
\end{equation}
Therefore,
\begin{equation}
    \sum_m \omega_i^m \| \mathbf{h}^m(v_i) - \mathbf{h}^*(v_i) \|
\le \frac{M}{2}\| \boldsymbol{\omega}_i - e^{I^*(v_i)} \|_1,
\end{equation}
and thus
\begin{equation}
\| \mathbf{h}^{\text{fused}}(v_i) - \mathbf{h}^*(v_i) \| \leq \frac{M}{2} \| \boldsymbol{\omega}_i - e^{I^*(v_i)} \|_1.
\end{equation}
Taking expectations:
\begin{equation}\label{eq:step3-distance-bound}
\mathbb{E}[ \| \mathbf{h}^{\text{fused}} - \mathbf{h}^* \| ] 
\leq \frac{M}{2} \mathbb{E}[ \| \boldsymbol{\omega} - e^{I^*} \|_1 ].
\end{equation}


By assumption, on each advantage region $\Omega_m$,
\begin{equation}
\mathbb{E}[ \| \boldsymbol{\omega} - e^m \|_1 \mid v_i \in \Omega_m ] \leq \delta.
\end{equation}
We now bound $\mathbb{E}[ \| \boldsymbol{\omega} - e^{I^*} \|_1 ]$ using the law of total expectation. Partition the node set $\mathcal{V}$ into the advantage regions $\{\Omega_m\}_{m \in \{E,H,S\}}$ and their complement:
\begin{equation}
\begin{aligned}
& \mathbb{E}[ \| \boldsymbol{\omega} - e^{I^*} \|_1 ] \\
& = \sum_{m \in \{E,H,S\}} \mu(\Omega_m) \,
    \mathbb{E}[ \| \boldsymbol{\omega} - e^{I^*} \|_1 \mid v_i \in \Omega_m ] \\
&\quad + \mu\big( \mathcal{V} \setminus \bigcup_{m \in \{E,H,S\}} \Omega_m \big) \,
   \mathbb{E}[ \| \boldsymbol{\omega} - e^{I^*} \|_1 \mid v_i \notin \bigcup_m \Omega_m ],
\end{aligned}
\end{equation}
where $\mu$ denotes the uniform probability measure over $\mathcal{V}$.
On each advantage region $\Omega_m$, we have $I^* = m$ by definition, so the first term becomes:
\begin{equation}
\begin{aligned}
& \sum_{m \in \{E,H,S\}} \mu(\Omega_m) \,
\mathbb{E}[ \| \boldsymbol{\omega} - e^m \|_1 \mid v \in \Omega_m ] \\
& \leq \sum_{m \in \{E,H,S\}} \mu(\Omega_m) \cdot \delta
= \delta \cdot \sum_{m \in \{E,H,S\}} \mu(\Omega_m)
\leq \delta.
\end{aligned}
\end{equation}
For the complement region, note that the $L^1$ distance between two probability vectors is bounded by 2. Therefore:
\begin{equation}
\mathbb{E}[ \| \boldsymbol{\omega} - e^{I^*} \|_1 \mid v \notin \bigcup_m \Omega_m ] \leq 2.
\end{equation}
Combining these bounds, we obtain:
\begin{equation}\label{eq:step4-expectation-bound}
\mathbb{E}[ \| \boldsymbol{\omega} - e^{I^*} \|_1 ]
\leq \delta + 2 \cdot \mu\big( \mathcal{V} \setminus \bigcup_{m \in \{E,H,S\}} \Omega_m \big).
\end{equation}


Combining \cref{eq:step2-lipschitz-bound}, \cref{eq:step3-distance-bound} and \cref{eq:step4-expectation-bound}, we have
\begin{equation}
\begin{aligned}
 \mathbb{E}[\ell(\mathbf{h}^{\text{fused}}, y)] 
 \leq \mathbb{E}[\ell^*] + \frac{LM}{2} \left( \delta + 2 \cdot \mu\left( \mathcal{V} \setminus \bigcup_m \Omega_m \right) \right).
\end{aligned}
\end{equation}
If the gating error $\delta$ is sufficiently small and the uncovered region $V\setminus\cup_{m}V_m$ has sufficiently small measure so that
\[
\frac{LM}{2} \left( \delta + 2 \cdot \mu\left( \mathcal{V} \setminus \bigcup_m \Omega_m \right) \right) < \eta_0,
\]
then from \cref{eq:step1-oracle-bound} we obtain the strict inequality
\[
\mathbb{E}[\ell(\mathbf{h}^{\text{fused}}, y)] < \min_{m \in \{E,H,S\}} \mathbb{E}[\ell^m].
\]

\end{proof}

\subsection{Proof of \cref{thm:Monotonic}}\label{Proof_monotonic}

\begin{proof}

The KL divergence is defined as:
\begin{equation}
\mathrm{KL}(\mathbf{\omega}^* \| \mathbf{\omega}) = \sum_{m \in \{E, H, S\}} \omega^{*m} \log \frac{\omega^{*m}}{\omega^{m}}.
\end{equation}
During minimization of this loss, the simplex constraint $\mathbf{\omega} \in \Delta^2 = \{ (\omega^E, \omega^H, \omega^S) \mid \sum_{m \in \{E, H, S\}} \omega^m = 1, \omega^m \ge 0 \}$ must be strictly satisfied. Computing the gradient of the KL divergence with respect to $w_m$ without accounting for the simplex constraint can lead to updates that violate $\sum_{m \in \{E, H, S\}} \omega^m = 1$.

To enforce this constraint, the Lagrange multiplier method \cite{bertsekas2014constrained} is employed. The Lagrangian is constructed as:
\begin{equation}
\mathcal{L}(\mathbf{\omega}, \lambda) = \mathrm{KL}(\mathbf{\omega}^* \| \mathbf{\omega}) + \lambda \left( \sum_{m \in \{E, H, S\}} \omega^m - 1 \right).
\end{equation}
Taking the partial derivative with respect to $\omega^m$ and setting it to zero yields the optimality condition:
\begin{equation}
\frac{\partial \mathcal{L}}{\partial \omega^m} = -\frac{\omega^{*m}}{\omega^m} + \lambda = 0 \quad \Rightarrow \quad \omega^m = \frac{\omega^{*m}}{\lambda}.
\end{equation}
Applying the simplex constraint $\sum_{m \in \{E, H, S\}} \omega^m = 1$ gives $\lambda = 1$, and thus the unique global minimizer is $\mathbf{\omega} = \mathbf{\omega}^*$.


Without loss of generality, we focus our analysis on the positive curvature region $\Omega_+ = \{ v \in \mathcal{V} \mid \kappa(v) \geq \theta \}$. 

Define:
\begin{equation}
\begin{aligned}
& a(v) = \sigma\left( \frac{ -|\kappa(v)| + \theta }{ \eta } \right), \\
& b(v) = \sigma\left( \frac{ -\kappa(v) - \theta }{ \eta } \right), \\
& c(v) = \sigma\left( \frac{ \kappa(v) - \theta }{ \eta } \right).
\end{aligned}
\end{equation}
Then $\omega^{*S} = c(v) / (a(v) + b(v) + c(v))$. By the quotient rule:

\begin{equation}
\frac{\partial \omega^{*S}}{\partial \kappa(v)} 
= \frac{ 
    \begin{aligned}
    &\frac{\partial c(v)}{\partial \kappa(v)} (a(v) + b(v) + c(v)) \\
    &- c(v) \left( \frac{\partial a(v)}{\partial \kappa(v)} + \frac{\partial b(v)}{\partial \kappa(v)} + \frac{\partial c(v)}{\partial \kappa(v)} \right)
    \end{aligned}
  }{(a(v) + b(v) + c(v))^2}.
\end{equation}


For any $v \in \Omega_+$ (i.e., $\kappa(v) \geq \theta > 0$), the signs of the partial derivatives are:
\begin{equation}
\begin{cases}
\dfrac{\partial a(v)}{\partial \kappa(v)} = -\dfrac{1}{\eta} \sigma'\left( \dfrac{ -\kappa(v) + \theta }{ \eta } \right) < 0, \\
\dfrac{\partial b(v)}{\partial \kappa(v)} = -\dfrac{1}{\eta} \sigma'\left( \dfrac{ -\kappa(v) - \theta }{ \eta } \right) < 0, \\
\dfrac{\partial c(v)}{\partial \kappa(v)} = \dfrac{1}{\eta} \sigma'\left( \dfrac{ \kappa(v) - \theta }{ \eta } \right) > 0.
\end{cases}
\end{equation}

The numerator of the derivative can be rewritten as:
\begin{equation}
N = \underbrace{\frac{\partial c(v)}{\partial \kappa(v)} (a(v) + b(v)) - c(v) \left( \frac{\partial a(v)}{\partial \kappa(v)} + \frac{\partial b(v)}{\partial \kappa(v)} \right)}_{>0}.
\end{equation}
Since $\frac{\partial a(v)}{\partial \kappa(v)} + \frac{\partial b(v)}{\partial \kappa(v)} < 0$, the second term is positive. The first term is obviously also positive. It follows immediately that $N > 0$. The denominator is strictly positive, so $\frac{\partial \omega^{*S}}{\partial \kappa(v)} > 0$ for all $v \in \Omega_+$.


Let $f(\kappa) = \frac{\partial \omega^{*S}}{\partial \kappa}$ be the derivative function defined on the domain of curvature. Given that the graph is finite, the curvature values $\kappa(v)$ are naturally bounded. We thus consider $f(\kappa)$ on the compact set $[\theta, M]$, where $M = \max_{v \in \mathcal{V}} \kappa(v)$ represents the maximum observed curvature. 

Since $f(\kappa)$ is continuous and strictly positive ($f(\kappa) > 0$) for all $\kappa \in [\theta, M]$, the Extreme Value Theorem guarantees that it attains a positive minimum:
\begin{equation}
C_+ = \min_{\kappa \in [\theta, M]} f(\kappa) > 0.
\end{equation}
Consequently, we have $\frac{\partial \omega^{*S}}{\partial \kappa(v)} \geq C_+ > 0$ for all $v \in \Omega_+$.

At the convergence point, $\mathbf{\omega} = \mathbf{\omega}^*$. Since the gating network is a differentiable function, it follows that:
\[
\frac{\partial \omega^{S}}{\partial \kappa(v)} = \frac{\partial \omega^{*S}}{\partial \kappa(v)} \geq C_+ > 0.
\]
The proofs for the negative curvature region $\Omega_- = \{ v \in \mathcal{V} \mid \kappa(v) \leq -\theta \}$ and the near-zero curvature region $\Omega_0 = \{ v \in \mathcal{V} \mid |\kappa(v)| \leq \theta \}$ are analogous, yielding constants $C_- > 0$ and $C_0 > 0$, respectively.
\end{proof}

\subsection{Proof of \cref{thm:Mutual_information}}\label{Proof_mutual}

\begin{proof}
We establish the mutual information lower bound for the ORC-guided contrastive learning objective. The proof follows the standard InfoNCE variational framework \cite{oord2018representation, poole2019variational} while explicitly incorporating our curvature-driven construction.


Let $v$ be a node uniformly sampled from $\mathcal{V}$. Define:
\begin{itemize}
    \item $\kappa = \kappa(v)$: the node-level ORC,
    \item $\mathbf{h}^{\text{fused}} = \mathbf{h}^{\text{fused}}(v)$: the fused representation,
    \item $\mathbf{h}^* = \mathbf{h}^*(\kappa)$: the ORC-guided geometric prior embedding (positive sample).
\end{itemize}
Let $\mathbf{h}^-_1, \dots, \mathbf{h}^-_K \stackrel{\text{i.i.d.}}{\sim} p(\mathbf{h}^*)$ be $K$ negative samples independently drawn from the marginal distribution of positive embeddings.


Form the candidate set $\mathcal{C} = \{\mathbf{h}^*, \mathbf{h}^-_1, \dots, \mathbf{h}^-_K\}$ of size $K+1$. Note that the subset of negatives $\{\mathbf{h}^-_1, \dots, \mathbf{h}^-_K\}$ corresponds to the set $\mathcal{N}^{-}(v)$ defined in the main text. Introduce a discrete index random variable $J \in \{0,1,\dots,K\}$ indicating which candidate is the true positive; $J=0$ corresponds to $\mathbf{h}^*$. Under the data generation process, we assume a uniform prior:
\begin{equation}\label{Proof_mutual_step2_1}
p(J = j \mid \mathcal{C}) = \frac{1}{K+1}, \quad j = 0,\dots,K.
\end{equation}
Define the model's predictive distribution via a softmax over similarity scores:
\begin{equation}\label{Proof_mutual_step2_2}
q_\theta(j \mid \mathbf{h}^{\text{fused}}, \mathcal{C}) 
= \frac{\exp\big(s(\mathbf{h}^{\text{fused}}, c_j)/\tau_c\big)}
{\sum_{\ell=0}^{K} \exp\big(s(\mathbf{h}^{\text{fused}}, c_\ell)/\tau_c\big)},
\end{equation}
where $c_0 = \mathbf{h}^*$, $c_j = \mathbf{h}^-_j$ for $j \geq 1$, $s(u,v) = u^\top v/(\|u\|\|v\|)$ is cosine similarity, and $\tau_c$ is the temperature parameter.

In training, we construct the candidate set such that the true positive always occupies the position $j=0$, i.e., $J=0$ in the data. 
The InfoNCE loss can then be written as the expected negative log-likelihood of the true index:
\begin{equation}\label{Proof_mutual_step2_3}
\mathcal{L}_{\text{contr}} 
= -\mathbb{E}_{(\mathbf{h}^{\text{fused}}, \mathcal{C})}\big[ \log q_\theta(J=0 \mid \mathbf{h}^{\text{fused}}, \mathcal{C}) \big]. 
\end{equation}


Consider the conditional mutual information between $\mathbf{h}^{\text{fused}}$ and $J$ given $\mathcal{C}$:
\begin{equation}\label{Proof_mutual_step3_1}
I(\mathbf{h}^{\text{fused}}; J \mid \mathcal{C})
=
\mathbb{E}_{p(\mathbf{h}^{\text{fused}}, J, \mathcal{C})}
\Big[
\log\frac{p(J \mid \mathbf{h}^{\text{fused}}, \mathcal{C})}{p(J \mid \mathcal{C})}
\Big]. 
\end{equation}

We now introduce the variational distribution $q_\theta(J \mid \mathbf{h}^{\text{fused}}, \mathcal{C})$, then add and subtract $\log q_\theta(J \mid \mathbf{h}^{\text{fused}}, \mathcal{C})$ inside the expectation:
\begin{equation}\label{Proof_mutual_step3_2}
\begin{aligned}
& I(\mathbf{h}^{\text{fused}}; J \mid \mathcal{C}) \\
&=
\mathbb{E}
\Big[
\log\frac{p(J \mid \mathbf{h}^{\text{fused}}, \mathcal{C})}{p(J \mid \mathcal{C})}
\Big]\\
&=
\mathbb{E}
\Big[
\log\frac{p(J \mid \mathbf{h}^{\text{fused}}, \mathcal{C})}{q_\theta(J \mid \mathbf{h}^{\text{fused}}, \mathcal{C})}
+
\log\frac{q_\theta(J \mid \mathbf{h}^{\text{fused}}, \mathcal{C})}{p(J \mid \mathcal{C})}
\Big]\\
&=
\underbrace{
\mathbb{E}
\Big[
\log\frac{q_\theta(J \mid \mathbf{h}^{\text{fused}}, \mathcal{C})}{p(J \mid \mathcal{C})}
\Big]}_{\text{variational term}}
+
\underbrace{
\mathbb{E}
\Big[
\log\frac{p(J \mid \mathbf{h}^{\text{fused}}, \mathcal{C})}{q_\theta(J \mid \mathbf{h}^{\text{fused}}, \mathcal{C})}
\Big]}_{\text{KL term}}.
\end{aligned}
\end{equation}

According to the definition of conditional expectation and the law of total expectation,  
the second term can be rewritten as a two-layer expectation over $(\mathbf{h}^{\text{fused}}, \mathcal{C})$ and $J$:

\begin{equation}\label{Proof_mutual_step3_3}
\begin{aligned}
&\mathbb{E} \left[ \log \frac{p(J \mid \mathbf{h}^{\text{fused}}, \mathcal{C})}{q_\theta(J \mid \mathbf{h}^{\text{fused}}, \mathcal{C})} \right] \\
&= \mathbb{E}_{\mathbf{h}^{\text{fused}}, \mathcal{C}} \left[ \mathbb{E}_{J \sim p(\cdot \mid \mathbf{h}^{\text{fused}}, \mathcal{C})} \left[ \log \frac{p(J \mid \mathbf{h}^{\text{fused}}, \mathcal{C})}{q_\theta(J \mid \mathbf{h}^{\text{fused}}, \mathcal{C})} \right] \right]. 
\end{aligned}
\end{equation}

The inner expectation is the KL divergence:
\begin{equation}\label{Proof_mutual_step3_4}
\begin{aligned}
& \mathbb{E}_{J \sim p(\cdot \mid \mathbf{h}^{\text{fused}}, \mathcal{C})}
[
\log \tfrac{p(J \mid \mathbf{h}^{\text{fused}}, \mathcal{C})}{q_\theta(J \mid \mathbf{h}^{\text{fused}}, \mathcal{C})}
] \\
& =
\mathrm{KL}\left(p(\cdot \mid \mathbf{h}^{\text{fused}}, \mathcal{C}) \,\|\, q_\theta(\cdot \mid \mathbf{h}^{\text{fused}}, \mathcal{C})\right).
\end{aligned}
\end{equation}

Hence:

\begin{equation}\label{Proof_mutual_step3_5}
\begin{aligned}
& \mathbb{E}
\Big[
\log\frac{p(J \mid \mathbf{h}^{\text{fused}}, \mathcal{C})}{q_\theta(J \mid \mathbf{h}^{\text{fused}}, \mathcal{C})}
\Big] \\
& =
\mathbb{E}_{\mathbf{h}^{\text{fused}}, \mathcal{C}}
\Big[
\mathrm{KL}\big(p(\cdot \mid \mathbf{h}^{\text{fused}}, \mathcal{C}) \,\|\, q_\theta(\cdot \mid \mathbf{h}^{\text{fused}}, \mathcal{C})\big)
\Big]
\;\ge\; 0,
\end{aligned}
\end{equation}

which yields the inequality
\begin{equation}\label{Proof_mutual_step3_6}
I(\mathbf{h}^{\text{fused}}; J \mid \mathcal{C})
\;\ge\;
\mathbb{E}
\Big[
\log\frac{q_\theta(J \mid \mathbf{h}^{\text{fused}}, \mathcal{C})}{p(J \mid \mathcal{C})}
\Big]. 
\end{equation}

Next, we expand the right-hand side:
\begin{equation}\label{Proof_mutual_step3_7}
\begin{aligned}
& \mathbb{E}
\Big[
\log\frac{q_\theta(J \mid \mathbf{h}^{\text{fused}}, \mathcal{C})}{p(J \mid \mathcal{C})}
\Big] \\
&=
\mathbb{E}
\Big[
\log q_\theta(J \mid \mathbf{h}^{\text{fused}}, \mathcal{C})
-
\log p(J \mid \mathcal{C})
\Big]\\
&=
\mathbb{E}\big[\log q_\theta(J \mid \mathbf{h}^{\text{fused}}, \mathcal{C})\big]
-
\mathbb{E}\big[\log p(J \mid \mathcal{C})\big].
\end{aligned}
\end{equation}

Using the uniform prior assumption in \cref{Proof_mutual_step2_1}, we have
\[
p(J \mid \mathcal{C}) = \frac{1}{K+1}
\quad\Rightarrow\quad
\log p(J \mid \mathcal{C}) = -\log(K+1),
\]
so that
\begin{equation}\label{Proof_mutual_step3_8}
-\mathbb{E}\big[\log p(J \mid \mathcal{C})\big] = \log(K+1). 
\end{equation}

Substituting \cref{Proof_mutual_step3_8} into \cref{Proof_mutual_step3_7} and then into \cref{Proof_mutual_step3_6} yields
\begin{equation}\label{Proof_mutual_step3_9}
\begin{aligned}
I(\mathbf{h}^{\text{fused}}; J \mid \mathcal{C})
&\ge
\mathbb{E}\big[\log q_\theta(J \mid \mathbf{h}^{\text{fused}}, \mathcal{C})\big]
+\log(K+1).
\end{aligned}
\end{equation}

On the other hand, from the definition of $\mathcal{L}_{\text{contr}}$ in \cref{Proof_mutual_step2_3}, we have
\begin{equation}\label{Proof_mutual_step3_10}
\begin{aligned}
\mathbb{E}
\big[\log q_\theta(J \mid \mathbf{h}^{\text{fused}}, \mathcal{C})\big]
= -\mathcal{L}_{\text{contr}}. 
\end{aligned}
\end{equation}

Combining \cref{Proof_mutual_step3_9} and \cref{Proof_mutual_step3_10}, we obtain
\begin{equation}\label{Proof_mutual_step3_11}
I(\mathbf{h}^{\text{fused}}; J \mid \mathcal{C})
\;\ge\;
\log(K+1) - \mathcal{L}_{\text{contr}}. 
\end{equation}


Finally, we connect the index $J$ to the curvature-driven geometric prior $\mathbf{h}^*(\kappa)$. Recall that the positive sample is defined as $\mathbf{h}^* = \mathbf{h}^*(\kappa)$, and the candidate set is $\mathcal{C} = \{\mathbf{h}^*, \mathbf{h}^-_1, \dots, \mathbf{h}^-_K\}$, where the negatives $\mathbf{h}^-_k$ are drawn i.i.d. from $p(\mathbf{h}^*)$. By assumption, these negatives are sampled independently of $\mathbf{h}^{\text{fused}}$, so $\mathcal{C}$ introduces no additional dependence between $\mathbf{h}^{\text{fused}}$ and $J$. Furthermore, given $\mathcal{C}$, the index $J$ is a deterministic function of $\mathbf{h}^*(\kappa)$, as it uniquely identifies the position of the positive sample within the set.

This yields a Markov structure
\[
\mathbf{h}^{\text{fused}} \;\longrightarrow\; \mathbf{h}^*(\kappa) \;\longrightarrow\; (\mathcal{C}, J),
\]
and by the data processing inequality~\cite{cover1999elements}, we have
\begin{equation}\label{Proof_mutual_step4_4}
I(\mathbf{h}^{\text{fused}}; \mathbf{h}^*(\kappa))
\;\ge\;
I(\mathbf{h}^{\text{fused}}; J \mid \mathcal{C}).
\end{equation}

Combining \cref{Proof_mutual_step3_11} and \cref{Proof_mutual_step4_4}, 
\begin{equation}
I\big(\mathbf{h}^{\text{fused}}; \mathbf{h}^*\big)
\;\ge\;
\log(K+1) - \mathcal{L}_{\text{contr}},
\end{equation}
which is exactly the claimed mutual information lower bound.
\end{proof}

\noindent{\textbf{Remark on Negative Sampling Assumptions:}}
The derivation of Theorem \ref{thm:Mutual_information} follows the standard assumption that negative samples are drawn i.i.d. from the marginal distribution $p(\mathbf{h}^*)$, ensuring that the candidate set $\mathcal{C}$ is conditionally independent of the anchor $\mathbf{h}^{\text{fused}}$. In our practical implementation, however, we employ a ``hard negative'' selection strategy (corresponding to $\mathcal{N}^{-}(v)$) based on curvature similarity to $\mathbf{h}^{\text{fused}}$. While this introduces a weak dependency between $\mathcal{C}$ and $\mathbf{h}^{\text{fused}}$ and deviates from the strict i.i.d. assumption, the theoretical bound remains a robust approximation for two reasons:
\begin{itemize}
    \item[(i)] The InfoNCE objective continues to provide a valid lower bound on mutual information when the dependency is weak or bounded, as discussed in recent analyses of contrastive learning~\cite{wang2020understanding}.
    \item[(ii)] Empirically, curvature-aware hard negatives enhance the discriminative power of the representations, thereby improving the empirical tightness of the bound under the same geometric consistency constraint.
\end{itemize}
Thus, while our practical sampling strategy modifies the idealized assumption, it preserves the underlying principle of mutual information maximization while strengthening the practical efficacy of representation learning.

\section{Algorithm}\label{Algorithm}

To provide a clear roadmap for our implementation, we delineate the training procedure of GeoMoE in Algorithm \ref{alg:geomoe}. The process begins with the offline precomputation of Ollivier-Ricci curvature (ORC), which serves as a static topological prior for geometric routing. During each training iteration, the model concurrently generates node representations among multiple manifold experts and dynamically determines the gating weights. The optimization follows a multi-objective strategy, where geometric consistency is enforced via $\mathcal{L}_{\text{align}}$ and $\mathcal{L}_{\text{contr}}$ alongside the loss of primary tasks $\mathcal{L}_{\text{task}}$. This structured approach ensures that the model parameters $\Theta$ are refined to align with the underlying graph topology effectively.

\begin{algorithm}[tb]
\caption{Training Process of GeoMoE}
\label{alg:geomoe}
\begin{algorithmic}[1]
\renewcommand{\algorithmicrequire}{\textbf{Input:}}
\renewcommand{\algorithmicensure}{\textbf{Output:}}

\REQUIRE Graph $\mathcal{G} = (\mathcal{V}, \mathcal{E})$, node features $\mathbf{X}$, adjacency matrix $\mathbf{A}$, labels $\mathbf{Y}$, ORC threshold $\theta$, smoothing factor $\eta$, loss weights $\alpha, \beta, \gamma$, max epochs $T$
\ENSURE Model parameters $\Theta$ and fused node embeddings $\{\mathbf{h}_{v_i}^{\text{fused}}\}_{v_i \in \mathcal{V}}$

\STATE Precompute Ollivier-Ricci curvature $k(v_i)$ for all $v_i \in \mathcal{V}$ and initialize $\Theta$ randomly
\FOR{$t = 1$ \textbf{to} $T$}
    \STATE Compute expert embeddings $\mathbf{h}^E, \mathbf{h}^H, \mathbf{h}^S$ via \cref{eq: GNN network};
    \STATE Calculate gating weights $\mathbf{w}$ via \cref{eq: gating weight};
    \STATE Obtain fused representations $\mathbf{h}^{\text{fused}}$ via \cref{eq: fused embedding};
    \STATE Compute gating alignment loss $\mathcal{L}_{\text{align}}$ via \cref{eq: alignment loss};
    \STATE Construct ORC-guided contrastive pairs and compute $\mathcal{L}_{\text{contr}}$ via \cref{eq: contrastive loss};
    \STATE Evaluate supervised loss $\mathcal{L}_{\text{task}}$ and formulate total loss $\mathcal{L}_{\text{total}}$ via \cref{eq: total loss};
    \STATE Update $\Theta$ via gradient descent.
\ENDFOR
\end{algorithmic}
\end{algorithm}

\section{Complexity and Scalability Analysis}\label{app:complexity}

In this section, we provide a detailed analysis of the computational complexity of GeoMoE. We distinguish between the \textit{offline geometric pre-computation} and the \textit{online model training}, demonstrating that our framework is efficient and scalable to large graph datasets.

\subsection{Offline Geometric Pre-computation}

The core geometric signal in our framework is the Ollivier-Ricci Curvature (ORC). The bottleneck lies in computing the 1-Wasserstein distance $W_1(\mu_u, \mu_v)$ for each edge $(u,v) \in \mathcal{E}$.

\paragraph{Exact Computation.} 
Exact computation of $W_1$ involves solving a Linear Programming (LP) problem. For a graph with maximum degree $\Delta$, the transport plan has size roughly $\Delta \times \Delta$. Using standard interior-point methods, the worst-case complexity per edge is $\mathcal{O}(\Delta^3 \log \Delta)$. For the entire graph, the complexity is:
\begin{equation}
    \mathcal{T}_{\text{exact}} = \mathcal{O}(|\mathcal{E}| \cdot \Delta^3 \log \Delta).
\end{equation}
While feasible for the datasets used in our main experiments (e.g., PubMed, roughly 20k nodes), this becomes prohibitive for million-scale graphs.

\paragraph{Scalable Approximations.} 
To ensure scalability, GeoMoE supports multiple approximation strategies: 

\begin{itemize}
    \item \textbf{Sinkhorn Algorithm \cite{cuturi2013sinkhorn}:} By regularizing the optimal transport problem, the Sinkhorn algorithm converts the LP into iterative matrix scaling operations. The complexity per edge is $\mathcal{O}(L \cdot \Delta^2)$, where $L$ is the number of iterations (typically small, e.g., $10-20$). This offers a significant speedup over exact LP.

    \item \textbf{Local Clustering Coefficient Bound \cite{jost2014ollivier}:}
    Jost and Liu derived a theoretical lower bound for ORC based on local clustering coefficients. For an edge $(u,v)$, the curvature is bounded by geometric properties related to the number of triangles (shared neighbors). Computing this bound relies on triangle enumeration, which typically costs $\mathcal{O}(\Delta^2)$ per node but is significantly faster than LP solvers in practice due to the sparsity of real-world graphs.
    
    \item \textbf{Neighborhood Overlap Approximation \cite{ni2019community}:} For maximum scalability, we approximate the transport cost using the Jaccard index of neighborhoods. The intersection of two sorted neighbor lists of size $d_u$ and $d_v$ takes $\mathcal{O}(d_u + d_v)$. Summing over all edges, the total complexity is proportional to the sum of squared degrees, bounded by:
    \begin{equation}
    \mathcal{T}_{approx} = \mathcal{O}(|\mathcal{E}|\cdot\Delta)
    \end{equation}
    
    In real-world sparse graphs where $\Delta \ll |\mathcal{V}|$, this approximation scales linearly with the number of edges, making it applicable to massive networks.
\end{itemize}

\subsection{Online Training Complexity}

During the training phase, GeoMoE decouples from the curvature calculation logic. The curvature values $\kappa(u,v)$ are stored as static edge attributes.

\begin{itemize}
    \item \textbf{Gating Network:} The curvature-guided gating alignment (Section 4.2) performs a simple lookup operation and an MLP projection. This incurs an $\mathcal{O}(1)$ cost per node/edge, negligible compared to feature aggregation.
    
    \item \textbf{Expert Computation:} The three experts (Euclidean, Hyperbolic, Spherical) operate in parallel. Assuming a standard message-passing scheme with hidden dimension $d$, the complexity is $\mathcal{O}(|\mathcal{E}|d + |\mathcal{V}|d^2)$. While Riemannian manifold operations (e.g., logarithmic maps) introduce a larger constant factor than Euclidean operations, they do not change the asymptotic complexity and are efficiently parallelized on GPUs. 
    
    \item \textbf{Overall Complexity:} The total training complexity per epoch is:
    \begin{equation}
        \mathcal{T}_{\text{train}} = \mathcal{O}(|\mathcal{E}| \cdot d + |\mathcal{V}| \cdot d^2).
    \end{equation}
\end{itemize}

Since $d$ is a small constant (e.g., 16 or 64), the training complexity is linear in the number of edges, matching standard GNN baselines such as GCN.

\subsection{Discussion on Approximation Validity}

One might question whether approximate curvature compromises the model's performance. Recent studies \cite{sun2023deepricci, feng2024graph} suggest that for graph learning tasks, the \textit{relative ranking} and \textit{sign} of the curvature (indicating tree-likeness vs. clique-likeness) are more critical than the exact numerical value usually. Since our gating mechanism (Equation 7) relies on thresholds and soft alignments, it is robust to the minor numerical perturbations introduced by approximation methods. Therefore, GeoMoE maintains its effectiveness even under scalable approximation regimes.

\section{Experiments}

\subsection{Experimental Settings}\label{Experimental Settings appendix}

\textbf{Datasets.} We evaluate our method on six benchmark datasets spanning diverse domains.
Cora, CiteSeer, and PubMed \cite{sen2008collective, namata2012query} are canonical citation networks, in which nodes correspond to scientific publications and citations are treated as undirected edges.
Photo \cite{shchur2018pitfalls} is derived from the Amazon co-purchase graph, comprising photography-related products as nodes, with edges reflecting frequent co-purchase patterns.
WikiCS \cite{mernyei2020wiki} constructs a hyperlink graph from Wikipedia, with nodes representing computer science articles and edges denoting internal hyperlinks; each article is annotated with one of ten subfield categories.
Airport \cite{chami2019hyperbolic} models the global air transportation system, where nodes are airports and edges reflect the existence of commercial airline routes between them. Detailed statistics for all datasets are provided in \cref{table:Dataset}

\textbf{Implementation Details.} In our experiments, we adopt the standard data splits for the three citation networks (Cora, CiteSeer, PubMed). 
For Photo and Airport, we use random splits of 60\%/20\%/20\% and 70\%/15\%/15\% for training, validation, and testing, respectively.
For WikiCS, we use the official fixed split provided by the dataset, specifically the first of its 20 predefined training/validation partitions along with the standard test set.
The grid search is performed over the following search space: Learning rate: [0.001, 0.005, 0.01, 0.02]; Dropout rate: [0.0, 0.1, 0.2, 0.3, 0.4, 0.5, 0.6]; weight decay: [0, 1e-4, 5e-4, 1e-3]; Number of hidden layers: [1, 2, 3]. To ensure statistical reliability, all results are reported as the average of 10 runs with random parameter initializations. We set the dimension of latent representation of all methods as 16. In our GeoMoE framework, we designate GCN and HGAT as the Euclidean and hyperbolic base experts, respectively. All experiments are implemented in Python using the PyTorch framework and conducted on an NVIDIA GeForce RTX 3090 GPU.

\begin{table}[t]
    \caption{Statistics of the experimental datasets.}
    \label{table:Dataset}
    \begin{tabular}{ccccc}
        \toprule
        Dataset & \#Nodes & \#Edges & \#Classes & \#Features \\
        \midrule
        Cora & 2,708 & 5,429 & 7 & 1,433 \\
        CiteSeer & 3,327 & 4,732 & 6 & 3,703 \\
        PubMed & 19,717 & 44,338 & 3 & 500 \\
        Photo & 7,650 & 119,081 & 8 & 745 \\
        WikiCS & 11,701 & 216,123 & 10 & 300 \\
        Airport & 3,188 & 18,631 & 4 & 4 \\       
        \bottomrule
    \end{tabular}
\end{table}

\begin{figure}[h!]
  \centering
  \includegraphics[width=0.95\columnwidth]{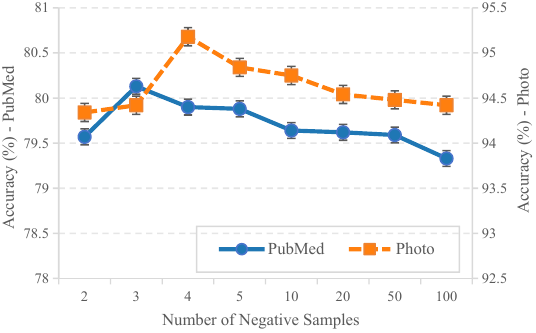}
  \caption{Impact of negative sample count $K$.}
  \label{fig:negative_samples}
\end{figure}

\subsection{Further Analysis}\label{app:Further Analysis}

\textbf{Impact of negative sample count.}
We analyze the impact of negative sample count $K$ on model performance across PubMed and Photo datasets. As shown in \cref{fig:negative_samples}, performance initially improves with increasing $K$, peaking at $K=3$ for PubMed and $K=4$ for Photo.
This demonstrates that our inter-node negative sampling strategy effectively enforces geometry-aware discrimination, with optimal regularization achieved using only a modest number of negatives.
However, we observe a marginal performance decay as the cardinality exceeds $10$, which can be attributed to the "false negative" phenomenon where topologically similar nodes are erroneously pushed apart in the embedding space. Overall, GeoMoE maintains robust performance within the range of $[3, 20]$, making it easy to tune in practice.

\end{document}